%% file: main.tex
\documentclass{article}
\usepackage{amsmath,amsbsy,amsfonts,amssymb,amsthm,dsfont}
\usepackage{natbib}

\setcitestyle{authoryear,round,citesep={;},aysep={,},yysep={;}}
\usepackage[colorlinks = true, linkcolor = blue, citecolor = cyan]{hyperref}

\renewcommand{\cite}[1]{\citep{#1}}

\usepackage{fullpage}
\usepackage{graphicx}
\usepackage{algorithm,algorithmic,mathtools}
\usepackage{color,cases}
\usepackage{tikz}
\usetikzlibrary{calc,shapes}
\usepackage{subfigure}
\usepackage{float}
\usepackage{hyperref}
\usepackage[many]{tcolorbox}   

\include{defs}
\newtheorem{cor}{Corollary}[section]
\newtheorem{lemma}{Lemma}[section]

\def\shownotes{1}  %set 1 to show author notes
\ifnum\shownotes=0
\newcommand{\authnote}[2]{{$\ll$\textsf{\footnotesize #1 notes: #2}$\gg$}}
\else
\newcommand{\authnote}[2]{}
\fi
\newcommand{\Anote}[1]{{\color{blue}\authnote{Andrej}{{#1}}}}

\title{Theoretical Limitations of Encoder-Decoder GAN architectures}
\author{Sanjeev Arora\thanks{Princeton University, Computer Science Department \texttt{arora@cs.princeton.edu}}, Andrej Risteski\thanks{Massachusets Institute of Technology, Applied Mathematics and IDSS \texttt{risteski@mit.edu}}, Yi Zhang\thanks{Princeton University, Computer Science Department \texttt{y.zhang@cs.princeton.edu}}}
\begin{document}
\maketitle
\begin{abstract}
Encoder-decoder GANs architectures (e.g., BiGAN and ALI) seek to add an \textquotedblleft inference\textquotedblright\ mechanism to the GANs setup, consisting of a small encoder deep net that maps data-points to their succinct encodings.  The intuition is that being forced to train an encoder alongside the usual generator forces the system to learn meaningful mappings from the code to the data-point and vice-versa, which should improve the learning of the target distribution and ameliorate mode-collapse.  It should also yield meaningful codes that are useful as  features for downstream tasks. 
The current paper shows rigorously that even on real-life distributions of images, the encode-decoder GAN training objectives (a) cannot prevent mode collapse; i.e. the objective can be near-optimal even when the generated distribution has low and finite support (b) cannot prevent learning meaningless codes for data -- essentially white noise.

Thus if encoder-decoder GANs do indeed work then 
it must be due to reasons as yet not understood, since the training objective can be low even for meaningless solutions.

Though the result statement may see reminiscent in spirit to the  ICML'17 paper~\cite{arora2017generalization}, the proof is novel.
\iffalse 
In what sense can GANs (Generative Adversarial Nets) actually be said to ``work''? The foundational paper of~\cite{goodfellow2014generative} suggested they in fact learn the target distribution, if they were given \textquotedblleft sufficiently large\textquotedblright\ deep nets, sample size, and computation time. A recent theoretical analysis in~\cite{arora2017generalization} raised doubts whether the same continues to hold when the discriminator has bounded size: it showed that the training objective can be close to optimum even if the generated distribution has very low support ---in other words, the training objective is unable to prevent {\em mode collapse}. Furthermore, though these worries may seem theoretical only (i.e. one might hope the training process for GANs somehow avoids these bad solutions), 
 the recent work \cite{arora2017gans} proposed a quantitative ``birthday paradox'' test, which suggests that even with state-of-the-art GAN architectures, this problem seems to be real. \fi

\end{abstract}

\section{Introduction}
\input{1-intro}

\input{4-ALI}

\section{Conclusions}

We have considered the theoretical shortcomings of encoder-decoder GAN architectures, which {\em a priori} seemed promising because they  target feature learning, a potentially simpler task than learning the full underlying distribution. At the same time, it was hoped that forcing the generator to learn \textquotedblleft meaningful\textquotedblright\ encodings of the image should improve distribution learning as well by ameliorating mode collapse. Our work suggests however that the learning objective alone does not guarantee such success; the  objective can be low even though the GAN has learnt  meaningless features which amount to noise. Furthermore, the learnt distribution can exhibit severe mode collapse, similarly as it does for the usual GAN setup.  

This theoretical problem arises from two causes: (a) Use of generative models that are highly expressive (multilayer nets) and which can therefore exhibit behavior unanticipated by the designer. (Obviously, this issue doesn't arise in less expressive models like mixtures of gaussians.) (b) Lack of any explicit probability calculations in the training. (Though of course this design decision  arose after the failure of earlier models that did rely on such explicit calculations.)

Our theoretical analysis points to gaps in current ways of reasoning about GANs, but it leaves open the possibility that encoder-decoder GANs do work in practice, with the training somehow avoiding the bad solutions exhibited here. 
%it must be due to reasons as yet not understood, since good value of the training objective does not predict good qualities of the solution. 
We hope our results will stimulate further theoretical and empirical study.
% hope that this continued investigation into possible modes of failure of various GAN setups will stimulate new ides on solving the open problem of how to change the training objectives so that we can avoid mode collapse.   

%Finally, one  should consider the possibility that  This needs further theoretical and empirical exploration.
 
%\section*{Acknowledgements}
%We thank Rong Ge and Cyril Zhang for helpful discussions and Ian Goodfellow for his comments on the manuscript. This research was supported by the National Science Foundation (NSF), Office of Naval Research (ONR), and the Simons Foundation.
 
\bibliography{draft}
\bibliographystyle{iclr2018_conference}
\newpage
\appendix 
\input{appendix}

\end{document}

%% file: defs.tex
\newtheorem{theorem}{Theorem}

\DeclareMathOperator*\E{\mathbb{E}}

%% file: 1-intro.tex
%auto-ignore
%!TEX root=main.tex

Generative Adversarial Networks (GANs) have emerged as formidable new method in unsupervised learning, often learning to generate images that are visually more appealing than the ones generated using other unsupervised learning methods (e.g. variational autoencoders, generative stochastic networks, deep Boltzmann machines etc.)  
However, in contrast to many previous unsupervised learning methods, GANs do not provide an estimate of a measure of distributional fit (e.g. likelihood calculation on heldout data). Then there is no obvious guarantee of generalization at the end, and the persistent fear has been of mode-collapse (the simplest being that the generator network memorizes training examples). Additionally, standard GAN frameworks may not provide any meaningful features (or latent representations) for downstream tasks---which  is often the goal of unsupervised learning. 

A theoretical study of GANs was initiated in the seminal work of \citep{goodfellow2014generative}, which proved that when  sample sizes, generator sizes, and discriminator sizes are all unbounded (i.e., infinite) then the generator converges to the true data distribution. But a recent theoretical analysis of GANs with finite training samples and finite discriminator size \cite{arora2017generalization} reached a different conclusion: it was proved that the training objective can be close to optimal while at the same time the generator is far from having actually learnt the distribution. Concretely, \cite{arora2017generalization} show that if the discriminator has size bounded by $p$, the training objective can be $\epsilon$ close to optimal even though the output distribution is supported on only $O(p\log p/\epsilon^2)$ images. By contrast one imagines that the target distribution usually must have very large support : the set of all possible images of human faces (a frequent setting in GANs work) should effectively involve all combinations  of hair color/style, facial features, complexion, expression, pose, lighting, race, etc., and thus the possible set of images of faces approaches infinity.

Of course, such a theoretical analysis does not in principle preclude the possibility that the training process of GAN's somehow avoids such low-support solutions -- similarly to how SGD seems to avoid bad local optima in the (supervised) training of feedforward neural networks. Clearly, the issue needed to be studied empirically as well. A recent paper~\cite{arora2017gans} proposed a \emph{birthday paradox} test to probe the diversity of trained GANs.  The birthday paradox states that if we are sampling uniformly at random from a set of support $N$, we will start seeing collisions (i.e. repeated samples of the same element) after sampling about $\sqrt{N}$ elements. This is adapted to the continuous regime by sampling $s$ samples from the generator, finding the 20 most similar images using some automated measure of similarity, and visually inspecting these 20 images for duplicates. (Note, the last step ensures ``no false positives'': in order to find a ``duplicate'' a visual inspection of it must have occured.) If we find a duplicate, this suggest the support of the distribution cannot be more than about $s^2$. The results of the birthday paradox test in \cite{arora2017gans} suggest the low-support solutions aren't a merely theoretical issue, but do actually occur even in practically trained GANs, which suffer from mode-collapse to various degrees.

Encoder-decoder frameworks like BiGAN~\citep{donahue2017adversarial} and Adversarially Learned Inference or ALI \citep{dumoulin2017adversarially} were recently proposed  towards fixing both the issue of mode collapse, and the lack of features in the standard GAN setup. Inspired by autoencoder models, they force the generative model to learn an inference mechanism as well as a generative mechanism.   The hope is that the  encoding mechanism \textquotedblleft inverts\textquotedblright\ the generator and thus forces the generator to learn meaningful featurizations of date. It has been suggested that the constraint of learning ``meaningful features'' will also help the mode collapse problem: \citep{dumoulin2017adversarially} report experiments on 2-dimensional mixtures of Gaussians suggesting this is indeed the case.
More promisingly, the theoretical result of \cite{arora2017generalization} also doesn't seem to extend to encoder-decoder architectures. %
 % only to standard GAN training objectives (including JS and Wasserstein), but inherently fails for encoder-decoder  setups 
  Thus it was an open problem whether  encoder-decoder GANs suffer from the same theoretical limitations as standard GANs. 
  (We note that the above-mentioned emprical study~\cite{arora2017gans} did report that ALI training also suffers from mode collapse, although it seems slightly better than other GANs setups in this regard.) 

The current paper provides theoretical analysis showing that encoder-decoder training objectives cannot avoid mode collapse even for very realistic target distributions (basically, real-life images) and they cannot enforce learning of meaningful codes/features as well. In fact, a close-to-optimum solution to the encoder-decoder optimization can be achieved by an inference mechanism that essentially extracts white noise from the images, and
where the generator produces a distribution of finite support whose size is moderate (sub-quadratic in the discriminator size).
The proof is novel, as explained below.

%% file: 4-ALI.tex
\section{Preliminaries}
We recall the Adversarial Feature Learning (BiGAN) setup from \cite{donahue2017adversarial}. For concreteness we assume the setup is being trained on the distribution of real-life images.

The ``generative'' player consists of two parts: a generator $G$ and an encoder $E$. The generator takes as input a latent variable $z$ and produces a sample $G(z)$ that is its attempt to output a realistic image. The encoder takes as input an actual image $x$ and produces $E(x)$, which is a guess for the latent variable that can generate $x$. 

The underlying assumption is that images come from an unknown manifold whose dimension is much lower than the number of pixels. Latent variables correspond to image representations on this manifold and Encoder (resp., generator) map from images to manifold points (resp., manifold points to images).  Then there exists a distribution
$p(z,x)$,  where $p$ is the joint distribution of the latent variables and data. The goal of the training is to yield $E, G$ such that $G(z)$ is distributed as $p(x|z)$, the true generator distribution; and $E(x)$ is distributed as $p(z|x)$, the true encoder distribution.  The hope is that the trained encoder-decoder pairs are such that both $(z, G(z))$ and $(E(x), x)$ are equal to
$p(z, x)$. 

In the older encoder-decoder frameworks, $E$ and $G$ would be trained jointly using variational inference. But the GAN setup avoids any explicit probability calculations, and instead uses an adversarial discriminator who can be asked (i.e., trained) to distinguish between two given distributions. % are considered similar if 
Thus the goal of the ``generative'' player is to convince the discriminator that these two distributions $(z, G(z))$ and $(E(x), x)$ are the same, where $z$ is a random seed and $x$ is a random image. The discriminator is being trained to distinguish between them. 

Using usual min-max formalism for adversarial training, the BiGAN objective is written as: 
\begin{equation} \min_{G, E} \max_{D} |\E_{x \sim \hat{\mu}} \phi(D(x, E(x))) - \E_{z \sim \hat{\nu}} \phi(D(G(z), z))|
\label{eqn:ali}
\end{equation} 
where $\hat{\mu}$ is the empirical distribution over data samples $x$; $\hat{\nu}$ is a distribution over random ``seeds'' for the latent variables: typically sampled from a simple distribution like a standard Gaussian; and $\phi$ is a concave ``measuring'' function. (The standard choice is $\log$, though other options have been proposed in the literature.) 
For our purposes, we will assume that $\phi$ outputs values in the range $[-\Delta, \Delta], \Delta \geq 1$, and is $L_{\phi}$-Lipschitz.

As mentioned, this objective leads to the target distribution being learnt, given enough capacity in the nets, samples, and training time. But the earlier analysis of~\cite{arora2017generalization} showed that finite-capacity discriminators behave very differently from infinite capacity discriminators, in that they cannot prevent the learnt distribution from exhibiting mode-collapse. But their proof cannot handle the encoder-decoder framework, and the obstacle is nontrivial. Their argument %of~\cite{arora2017generalization} seems unable to apply to this setting.  %In \cite{arora2017generalization}, by 
%It is 
a simple concentration/epsilon-net argument showing that the discriminator of capacity $p$ cannot distinguish between  a generator that samples from $\mu$ versus one that memorizes a subset of $\frac{p \log p}{\epsilon^2}$ random images in $\mu$ and outputs one randomly from this subset. By contrast, in the current setting we need to say what happens with the encoder. A big obstacle is that fairness requires that the encoder net should be smaller than the discriminator, so that discriminator could (in principle, if needed) learn to compute $E$ by itself. Thus in particular the proof must end up describing an explicit small encoder, which seems very difficult. (No such explicit description is known, and generative models are only an approximation.) This difficulty is cleverly circumvented in the argument below by making the encoder map images to random noise extracted from the image.

\section{Limitations of Encoder-Decoder GAN architectures}
\label{s:ALI}
For ease of exposition we will refer to the data distribution $\mu$ as the {\em image distribution}. The proof becomes more elegant if we assume that $\mu$ consists of images that have been noised ---concretely, think of replacing every $100$th pixel by Gaussian noise. %(Formally speaking, total number of pixels that have to be noised is the dimension of the latent space, which is typically extremely small compared to the number of pixels). 
Such noised  images would of course look fine to our eyes, and we would expect the learnt encoder/decoder to not be affected by this noise. For concreteness, we will take the seed/code distribution $\nu$ to be a spherical zero-mean Gaussian (in an arbitrary dimension and with an arbitrary variance). \footnote{
The proof can be extended to non-noised inputs, by assuming that natural images have an innate stochasticity that can be extracted
by a small net to get a few statistically random bits. We chose not to write it that way because it requires making a novel assumption about images. The proof also extends to more general code distributions than Gaussian.} 

Furthermore, we will assume that $\mbox{Domain}(\mu) = \mathbb{R}^d$, $\mbox{Domain}(\nu) = \mathbb{R}^{\tilde{d}}$ with $\tilde{d} < d$ (we think of $\tilde{d} \ll d$, which is certainly the case in practice). As in~\cite{arora2017generalization} we assume that discriminators are $L$-lipschitz with respect to their trainable parameters, and the support size of the generator's distribution will depend upon this $L$ and the capacity $p$ (= number of parameters) of the discriminator. 
%The main theorem we show is as follows: 

\begin{theorem} [Main]  There exists a generator $G$ of support $\frac{p \Delta^2 \log^2 (p \Delta L L_{\phi}/\epsilon)}{\epsilon^2}$ and an encoder $E$ with at most $\tilde{d}$ non-zero weights, s.t. for all discriminators $D$ that are $L$-Lipschitz and have capacity less than $p$, it holds that \\ 
$$|\E_{x \sim \mu} \phi(D(x, E(x))) - \E_{z \sim \nu} \phi(D(G(z), z))| \leq \epsilon $$
%and \\
%(ii) $|\E_{x \sim \hat{\mu}} \phi(D(x, G(x))) - \E_{x \sim \mu} \phi(D(x, G(x))) | \leq \epsilon,  |\E_{x \sim N(0,I)} \phi(D(E(z), z)) - \E_{z \sim N(0,I)} \phi(D(E(z), z)) | $ \\ 

\label{t:main} 
\end{theorem} 

%{\sc why did this switch to $C(x)$ from $E(x)$.}

The interpretation of the above theorem is as stated before: the encoder $E$ has very small complexity (we will subsequently specify it precisely and show it simply extracts noise from the input $x$); the generator $G$ is a small-support distribution (so presumably far from the true data distribution). Nevertheless, the value of the BiGAN objective is small.

\noindent{\em The precise noise model:}  Denoting by $\tilde{\mu}$ the distribution of unnnoised images, and $\nu$ the distribution of seeds/codes,  we define the distribution of noised images $\mu$ as the following distribution: to produce a sample in $\mu$  take a sample $\tilde{x}$ from  $\tilde{\mu}$ and $z$ from $\nu$ independently and output $x = \tilde{x} \circledast z$, which is defined as   
 \[
    x_i = \begin{cases}
        z_{\frac{i}{\lfloor \frac{d}{\tilde{d}}\rfloor}}, & \text{if } i \equiv 0 (\mbox{ mod } \lfloor \frac{d}{\tilde{d}} \rfloor) \\
        \tilde{x}_{i}, & \text{ otherwise }
        \end{cases}
  \]  
In other words,  set every $\lfloor \frac{d}{\tilde{d}} \rfloor$-th to one of the coordinates of $z$. In practical settings, $\tilde{d} \ll d$ is usually chosen, so noising $\tilde{d}$ coordinates out of all $d$ will have no visually noticeable effect. \footnote{Note the result itself doesn't require constraints on $d,\tilde{d}$ beyond $\tilde{d} < d$ -- the point we are stressing is merely that the noise model makes sense for practical choices of $\tilde{d}, d$.} %We refer to this new distribution $\mu$ as $\tilde{\mu} \circledast  \nu$.
%, where the operation $\circledast$ is defined as follows:  (The frequent argument is that small latent variable dimensions force a stronger ``bottleneck'' and thus extract more relevant information about the data distribution). 

\subsection{Proof Sketch, Theorem~\ref{t:main}}
(A full proof appears in the appendix.)
The main idea is to show the existence of the generator/encoder pair via a probabilistic construction that is shown to succeed with high probability.  

\begin{itemize}
\item {\em Encoder $E$}: The encoder just extracts the noise from the noised image (by selecting the relevant $\tilde{d}$ coordinates). Namely,
$E(\tilde{x} \circledast z) =z$. (So the {\em code} is just gaussian noise and has no meaningful content.) It's easy to see this can be captured using a ReLU network 
with $\tilde{d}$ weights: we can simply connect the $i$-th output to the ($i \lfloor \frac{d}{\tilde{d}}\rfloor$)-th input using an edge of weight 1.  

\item {\em Generator $G$:} This is designed to produce a distribution of support size $m := \frac{p \Delta^2 \log^2 (p \Delta L L_{\phi}/\epsilon)}{\epsilon^2}$.  
We first define a partition of $\mbox{Domain}(\nu) = \mathbb{R}^{\tilde{d}}$ into $m$ equal-measure blocks under $\nu$. %and $\mathcal{P}_{\mu}$ be an arbitrary partition of $\mu$ into $M$ equal-measure cells under $\mu$.
Next, we sample $m$ samples $x^*_1, x^*_2, \dots, x^*_m$ from the image distribution. Finally, for a sample $z$, we define $\mbox{ind}(z)$ to be the index of the block in the partition in which $z$ lies, and define the generator as $G(z) = x^*_{\mbox{ind}(z)} \circledast z $. Since the set of samples $x^*_i: i \in [m]$ is random, this specifies a \emph{distribution} over generators. We prove that with high probability, one of these generators satisfies the statement of Theorem~\ref{t:main}. Moreover, we show that such a generator can be easily implemented using a ReLU network of complexity $O(md)$ in the full version. %Theorem~\ref{t:representen}. 

\end{itemize}

The basic intuition of the proof is as follows. We will call a set $T$ of samples from $\nu$ \emph{non-colliding} if no two lie in the same block. Let $\mathcal{T}_{nc}$ be the distribution over non-colliding sets $\{z_1, z_2, \dots, z_m\}$, s.t. each $z_i$ is sampled independently from the conditional distribution of $\nu$ inside the i-th block of the partition. 
%uniform distribution over non-colliding sets of size $m$ from $\nu$.

First, we notice that under the distribution for $G$ we defined, it holds that 
$$\E_{x \sim \mu} \phi(D(x, E(x)))  = \E_{x \sim \tilde{\mu}, z \sim \nu} \phi(D(x \circledast z, z))  = \E_{G} \E_{z \sim \nu} \phi(G(z), z)$$ 
In other words, the ``expected'' encoder correctly matches the expectation of $\phi(D(x, E(x)))$, so that the discriminator is fooled. We want to show that  
$\E_{G} \E_{z \sim \nu} \phi(G(z), z)$ concentrates enough around this expectation, as a function of the randomness in $G$, so that we can say with high probability over the choice of $G$, $|\E_{x \sim \mu} \phi(D(x, E(x))) - \E_{z \sim \nu} \phi(G(z), z)|$ is small.  
We handle the concentration argument in two steps: 

%\begin{itemize} 
First, we note that we can calculate the expectation of $\phi(D(G(z), z))$ when $z \sim \nu$ by calculating the empirical expectation over $m$-sized non-colliding sets $T$ sampled according to $\mathcal{T}_{nc}$. Namely, as we show in the full version. %Lemma \ref{l:approximategood}:  
$$ \E_{z \sim \nu} \phi(D(G(z), z)) = \E_{T \sim \mathcal{T}_{nc}} \E_{z \sim T} \phi(D(G(z), z)) $$
This follows easily from the fact that all blocks in the partition have equal measure under $\nu$.

Thus, we have reduced our task to arguing about the concentration of $\E_{T \sim \mathcal{T}_{nc}} \E_{z \sim T} \phi(D(G(z), z))$ (viewed as a random variable in $G$).  
Towards this, we consider the random variable $\E_{z \sim T} \phi(D(G(z), z))$ as a function of the randomness in $G$ and $T$ both. Since $T$ is a non-colliding set of samples, we can write 
$$\E_{z \sim T} \phi(D(G(z), z)) = f(x^*_i, z_i, i \in [m])$$
for some function $f$, where the random variables $x^*_i$, $z_i$ are all mutually independent -- thus use McDiarmid's inequality to argue about the concentration of $f$ in terms of both $T$ and $G$. 

From this, we can use Markov's inequality to argue that all but an exponentially small (in $p$) fraction of encoders $G$ satisfy that: for all but an exponentially small (in $p$) fraction of non-colliding sets $T$, $|\E_{z \sim T} \phi(D(G(z), z)) - \E_{G} \E_{T \sim \mathcal{T}_{nc}}\E_{z \sim T} \phi(D(G(z), z))|$ is small. Note that this has to hold \emph{for all} discriminators $D$ -- so we need to additionally build an epsilon-net, and union bound over all discriminators, similarly as in \cite{arora2017generalization}. 
From this fact, it's easy to extrapolate that for such $G$, $$|\E_{T \sim \mathcal{T}_{nc}}\E_{z \sim T} \phi(D(G(z), z)) - \E_{G} \E_{T \sim \mathcal{T}_{nc}}\E_{z \sim T} \phi(D(G(z), z))| $$ 
is small, as we want. The details are provided in the full version.%Lemma~\ref{l:generatorconcentration} in the Appendix. 
%\end{itemize}

% The proof of correctness uses a covering number argument. For any fixed discriminator....

%{\sc put in a paragraph here with buzzwords like concentration bound etc. such that an expert should be able to understand.}

%\section{{\sc Rest is in appendix}}

%% file: appendix.tex
\section{Technical proofs}

We recall the basic notation from the main part: the image distribution will be denoted as $\mu$, and the code/seed distribution as $\nu$, which we assume is a spherical Gaussian. \Anote{Can be removed, only necessary for partition.} For concreteness, we assumed the domain of $\tilde{\mu}$ is $\mathbb{R}^d$ and the domain of $\nu$ is $\mathbb{R}^{\tilde{d}}$ with $\tilde{d} \leq d$. (As we said, we are thinking of $\tilde{d} \ll d$.)

%As described above, the min (generative) player consists of a decoder $G_z: \mathbb{R}^d \to \mathbb{R}^{\tilde{d}}$, an encoder $G_x: \mathbb{R}^{\tilde{d}} \to \mathbb{R}^d$, and the max (discriminator) player is a discriminator $D: \mathbb{R}^d \times \mathbb{R}^{\tilde{d}} \to [0,1]$, which we assume to be 1-Lipschitz, and have $p$ parameters. (We will also refer to this as the ``capacity'' of the discriminator.) $\phi$ is a concave measuring function, same as in \cite{arora2017generalization}. 
We also introduced the quantity $m :=  \frac{p \Delta^2 \log^2 (p \Delta L L_{\phi}/\epsilon)}{\epsilon^2}$.% and $M:= m^2$. %Also, for a vector $v \in \mathbb{S}_{d}$, let's denote by $T(v)$ the set of hyperplanes $\{x: \langle v, x\rangle = b\}$  [NB: $\mathbb{S}_d$ denotes the unit sphere in $d$ dimensions.] 

Before proving Theorem~\ref{t:main}, let's note that the claim can easily be made into a finite-sample version. Namely: 
\begin{cor} [Main, finite sample version] There exists a generator $G$ of support $m$, s.t.
if $\hat{\mu}$ is the uniform distribution over a training set $S$ of size at least $m$, and $\hat{\nu}$ is the uniform distribution over a sample $T$ from $\nu$ of size at least $m$, for all discriminators $D$ that are $L$-Lipschitz and have less than $p$ parameters, with probability $1- \exp(-\Omega(p))$ over the choice of training set $S$,$T$ we have: \ 
$$|\E_{x \sim \hat{\mu}} \phi(D(x, E(x))) - \E_{z \sim \hat{\nu}} \phi(D(G(z), z))| \leq \epsilon $$
%and \\
%(ii) $|\E_{x \sim \hat{\mu}} \phi(D(x, G(x))) - \E_{x \sim \mu} \phi(D(x, G(x))) | \leq \epsilon,  |\E_{x \sim N(0,I)} \phi(D(E(z), z)) - \E_{z \sim N(0,I)} \phi(D(E(z), z)) | $ \\ 

\label{c:main} 
\end{cor} 
\begin{proof}
As is noted in Theorem B.2 in \cite{arora2017generalization}, we can build a $\frac{\epsilon}{L L_{\phi}}$-net for the discriminators with a size bounded by $e^{p \log (L L_{\phi} p/\epsilon)}$. By Chernoff and union bounding over the points in the $\frac{\epsilon}{L L_{\phi}}$-net, with probability at least $1-\exp(-\Omega(p))$ over the choice of a training set $S$, we have
$$|\E_{x \sim \mu} \phi(D(x, E(x))) - \E_{x \sim \hat{\mu}} \phi(D(x, E(x))) | \leq \frac{\epsilon}{2}$$ 
for all discriminators $D$ with capacity at most $p$.  
Similarly, with probability at least $1-\exp(-\Omega(p))$ over the choice of a noise set $T$, 
$$|\E_{z \sim \nu} \phi(D(G(z), z)) - \E_{z \sim \hat{\nu}} \phi(D(G(z), z))| \leq \frac{\epsilon}{2}$$ 
Union bounding over these two events, we get the statement of the Corollary. 

\end{proof}

Spelling out the distribution over generators more explicitly, we will in fact show:   

\begin{theorem} [Main, more detailed] Let $G$ follow the distribution over generators defined in Section~\ref{s:ALI}. With probability $1 - \exp(-\Omega(p \log (\Delta/\epsilon)))$ over the choice of $G$, for all discriminators $D$ that are L-Lipschitz and have capacity bounded by $p$: \\ 
$$|\E_{x \sim \mu} \phi(D(x, E(x))) - \E_{z \sim \nu} \phi(D(G(z), z))| \leq \epsilon$$
%and \\
%(ii) $|\E_{x \sim \hat{\mu}} \phi(D(x, G(x))) - \E_{x \sim \mu} \phi(D(x, G(x))) | \leq \epsilon,  |\E_{x \sim N(0,I)} \phi(D(E(z), z)) - \E_{z \sim N(0,I)} \phi(D(E(z), z)) | $ \\ 

\label{t:maindetailed} 
\end{theorem} 

\subsection{Proof of the main claim} 

Let us recall we call a set $T$ of samples from $\nu$ \emph{non-colliding} if no two lie in the same block and we denoted $\mathcal{T}_{nc}$ to be the distribution over non-colliding sets $\{z_1, z_2, \dots, z_m\}$, s.t. each $z_i$ is sampled independently from the conditional distribution of $\nu$ inside the i-th block of the partition.

First, notice the following Lemma: 
\begin{lemma}[Reducing to expectations over non-colliding sets] Let $G$ be a fixed generator, and $D$ a fixed discriminator. Then,  
$$ \E_{z \sim \nu} \phi(D(G(z), z)) = \E_{T \sim \mathcal{T}_{nc}} \E_{z \sim T} \phi(D(G(z), z)) $$
\label{l:approximategood} 
\end{lemma} 
\begin{proof}
By definition, 
$$\E_{T \sim \mathcal{T}_{nc}} \E_{z \sim T} \phi(D(G(z), z)) = \frac{1}{m} \sum_{i=1}^m \E_{z_i \sim (\mathcal{T}_{nc})_i} \phi(D(G(z_i), z_i))$$ 
where $(\mathcal{T}_{nc})_i$ is the conditional distribution of $\nu$ in the $i$-th block of the partition. However, 
since the blocks form an equipartitioning, we have 
\begin{align*} \frac{1}{m} \sum_{i=1}^m \E_{z_i \sim (\mathcal{T}_{nc})_i} \phi(D(G(z_i), z_i)) &= \sum_{i=1}^m \Pr_{z \sim \nu}(z \mbox{ belongs to cell } i) \E_{z_i \sim (\mathcal{T}_{nc})_i} \phi(D(G(z_i), z_i)) \\
&= \E_{z \sim \nu} \phi(D(G(z), z)) \end{align*}

%$\mathcal{T}_{nc})_i$ and $\nu$ are equal in distribution, from which the claim follows. 
\end{proof} 

\begin{lemma}[Concentration of good generators] 
With probability $1 - \exp(-\Omega(p \log (\Delta/\epsilon)))$ over the choice of $G$, 
$$|\E_{T \sim \mathcal{T}_{nc}}\E_{z \sim T} \phi(D(G(z), z)) - \E_{G} \E_{T \sim \mathcal{T}_{nc}}\E_{z \sim T} \phi(D(G(z), z))| \leq \epsilon$$
for all discriminators $D$ of capacity at most $p$. 
\label{l:generatorconcentration} 
\end{lemma} 
\begin{proof}

Consider $\E_{z \sim T} \phi(D(G(z), z))$ as a random variable in $T \sim \mathcal{T}_{nc}$ and $G$ for a fixed $D$.
%By Claims (1) and (2) in Lemma~\ref{l:goodsets}, 
We can write 
$$\E_{z \sim T} \phi(D(G(z), z)) = f(x^*_i, z_i, i \in [m])$$ 
where the random variables $x^*_i$, $z_i$ are all mutually independent. 
Note that the arguments that $f$ is a function of are all independent variables, so we can apply McDiarmid's inequality. 
Towards that, let's denote by $z_{-i}$ the vector of all inputs to $f$, except for $z_i$. Notice that 
\begin{equation}|f(z_{-i}, z_i) - f(z_{-i}, \tilde{z}_i)| \leq \frac{1}{m}, \forall i \in [m]\label{eq:condmcd}\end{equation} 
(as changing $z_i$ to $\tilde{z}_i$ only affect one out of the $m$ terms in $E_{z \sim T} \phi(D(G(z), z))$). Analogously we have 
\begin{equation}|f(x^*_{-i}, x^*_i) - f(x^*_{-i}, \tilde{x}^*_i)| \leq \frac{1}{m}, \forall i \in [m]\label{eq:condmcd2}\end{equation} 
%Towards that, let's denote by $a_{-i}$ the vector of all inputs to $f$, except for $a_i$ \Anote{need to say this more precisely}. If we show 
%\begin{equation}|f(a_{-i}, a_i) - f(a_{-i}, \tilde{a}_i)| \leq \frac{3}{m}, \forall i \in [m]\label{eq:condmcd}\end{equation} 
%for any $\tilde{a}_i \neq a_i$, then it will follow that with probability $1 - \exp(-O(m\epsilon^2))$, 
%$$|f(x^*_i, z^*_i, a_i, x_i, z_i, i \in [m]) - \E_{T, G} [f(x^*_i, z^*_i, a_i, x_i, z_i, i \in [m])]| \leq \epsilon$$ 
%(NB: similar statements should be proven for $x_{-i}, x^*_{-i}$, etc. but those statements are trivial because $\phi(\cdot) \in [0,1]$, even with 1 instead of 3 on the RHS of \eqref{eq:condmcd}). Finally, the reason \eqref{eq:condmcd} holds is, writing out $f$ explicitly as an average over $m$ terms, the only terms in the sums that are affected by changing $a_i$ to $\tilde{a}_i$ are the final positions of $i$, $a_i$ and $\tilde{a}_i$ -- and each of those terms is bounded by 1 in absolute value. 
%\Anote{Expand the above somewhat.}

Denoting $R_{D,T,G} = f(x^*_i, z_i, i \in [m]) - \E_{T, G} [f(x^*_i, z_i, i \in [m])]$, by McDiarmid we get 
\begin{align*} \Pr_{T,G} (|R_{D,T,G}| \geq \epsilon) \leq \exp(-\Omega(m \epsilon^2)) \end{align*}  
Building a $\frac{\epsilon}{L L_{\phi}}$-net for the discriminators and union bounding over the points in the net, we get that $\Pr_{T,G} (\exists D, |R_{D,T,G}| \geq \epsilon/2) \leq \exp(-\Omega(p \log(\Delta/\epsilon)))$. 
On the other hand, we also have
\begin{align*} \Pr_{T,G} (\exists D,|R_{D,T,G}| \geq \epsilon/2) &= \E_{G}[\Pr_{T}(\exists D, |R_{D,T,G}| \geq \epsilon)]  \end{align*}  
so by Markov's inequality, with probability $1 - \exp(-\Omega(p \log(\Delta/\epsilon))$ over the choice of $G$, \Anote{1/2 const in O}, 
\begin{align*} \Pr_{T}(\exists D, |R_{D, T,G}| \geq \epsilon/2)] \leq \exp(-\Omega(p \log (\Delta/\epsilon))) \end{align*}  
%Union bounding over all $D$, this claim holds for all $D$ of capacity at most $p$ with probability $1-\exp(-p)$. 
%\Anote{Does this require $\epsilon^4$?}
Let us denote by $\tilde{T}(G)$ the sets $T$, s.t. $\forall D, |R_{D, T, G}| \leq \epsilon/2$. Then, with probability $1 - \exp(-\Omega(p \log (\Delta/\epsilon))))$ over the choice of $G$, we have that for all $D$ of capacity at most $p$,  
\begin{align*} &|\E_T\E_{z \sim T} \phi(D(G(z), z)) - \E_{G, T}\E_{z \sim T} \phi(D(G(z), z))|  \\
&\leq \int_{T \in \tilde{T}(G_x)} |\E_{z \sim T} \phi(D(G(z), z)) - \E_{G, T}\E_{z \sim T} \phi(D(G(z), z))| d\mathcal{T}_{nc}(T) \\
&+  \int_{T \notin \tilde{T}(G_x)} |\E_{z \sim T} \phi(D(G(z), z)) - \E_{G, T}\E_{z \sim T} \phi(D(G(z), z))| d\mathcal{T}_{nc}(T) \\
&\leq \epsilon/2 + \exp(-\Omega(p \log(\Delta/\epsilon))) \\
&\leq \epsilon/2 + O((\epsilon/\Delta)^p) \\
&\leq \epsilon \\
 \end{align*} 
which is what we want.

\end{proof} 

With these in mind, we can prove the main theorem: 

\begin{proof}[Proof of Theorem \ref{t:main}]

From Lemma \ref{l:approximategood} and Lemma \ref{l:generatorconcentration} we get that with probability $1 - \exp(-\Omega(p \log (\Delta/\epsilon)))$ over the choice of $G$, 
$$|\E_{z \sim \nu} \phi(D(G(z), z)) - \E_{G, T \sim \mathcal{T}_G} \phi(D(G(z), z))| \leq \epsilon  $$
for all discriminators $D$ of capacity at most $p$. On the other hand, it's easy to see that  
$$\E_{G, T \sim \mathcal{T}_G} \phi(D(G(z), z)) = \E_{x \sim \tilde{\mu}, z \sim \nu} \phi(D(x \circledast z, z)) $$ 
Furthermore, 
$$\E_{x \sim \mu} \phi(D(x, E(x)))  = \E_{x \sim \tilde{\mu}, z \sim \nu} \phi(D(x \circledast z, z)) $$ 
by the definition of $E(x)$, which proves the statement of the theorem. 
%Putting (1) and (2) together, we have that with probability $1 - \exp(-O(m \epsilon^2))$ over the choice of $G$,  

\end{proof}

\subsection{Representability results}
\label{s:representability}

In this section, we show that a generator $G$ of the type we described in the previous section can be implemented easily using a ReLU network. The encoder can be parametrized by the set $x^*_1, x^*_2, \dots, x^*_m$ of ``memorized'' training samples. The high-level overview of the network is rather obvious: we partition $\mathbb{R}^{\tilde{d}}$ into equal-measure blocks; subsequently, for a noise sample $z$, we determine which block it belongs to, and output the corresponding sample $x_i^*$, which is memorized in the weights of the network. 

%For this section, we will assume the domain of $\tilde{\mu}$ is $\mathbb{R}^d$ and the domain of $\nu$ is $\mathbb{R}^{\tilde{d}}$. 

\begin{theorem} Let $G$ be the generator determined by the samples $x_1^*, x_2^*, \dots,  x_m^* \in \mathbb{R}^d$, i.e. $G(z) = x^*_{\mbox{ind}(z)} \circledast z$. For any $\delta > 0$, there exists a ReLU network with $O(md)$ non-zero weights, which implements a function $\tilde{G}$, s.t.  $\|\tilde{G}(\nu) - G(\nu)\|_{TV} \leq \delta$, where $\|\cdot\|_{TV}$ denotes total variation distance.\footnote{The notation $G(\nu)$ simply denotes the distribution of $G(z)$, when $z \sim \nu$.}
\footnote{For smaller $\delta$, the weights get larger, so the bit-length to represent them gets larger. This is standard when representing step functions using ReLU, for instance see Lemma 3 in \cite{arora2017generalization}. For the purposes of our main result Theorem \ref{t:main}, it suffices to make $\delta = O(\frac{\epsilon}{\Delta})$, which translates into a weights on the order of $O(\frac{m\epsilon}{\Delta})$ -- which in turn translates into bit complexity of $O(\log (m \epsilon/\Delta))$ so this isn't an issue as well.}\label{t:representen}
\end{theorem}   
\begin{proof} 

The construction of the network is pictorially depicted on Figure~\ref{f:circuit}. We expand on each of the individual parts of the network. 

First, we show how to implement the partitioning into blocks. The easiest way to do this is to use the fact that the coordinates of a spherical Gaussian are independent, and simply partition each dimension separately into equimeasure blocks, depending on the value of $|z_i|$: the absolute value of the $i$-th coordinate. Concretely, 
without loss of generality, let's assume $m = k^{\tilde{d}}$ \Anote{This might blow up $m$ by factor of $\tilde{d}$. Can this be avoided?}, for some $k \in \mathbb{N}$. Let us denote by $\tau_i \in \mathbb{R}: i \in [k]$, $\tau_0 = 0$ the real-numbers, s.t. $\Pr_{z_i \sim \nu_i}[|z_i| \in (\tau_{i-1}, \tau_i)] = \frac{1}{k}$. \Anote{i.e. they partition the univariate Gaussian measure into $k$ pieces}.  
%$\frac{2\pi}{m} i$. 
We will associate to each $d$-tuple $(i_1, i_2, \dots, i_d) \in [k]^d$ the cell 
$\{z \in \mathbb{R}^{\tilde{d}}: |z_j| \in (\tau_{i_j-1}, \tau_{i_j}), \forall j \in [\tilde{d}]\} $.  
These blocks clearly equipartition $\mathbb{R}^{\tilde{d}}$ with respect to the Gaussian measure. 

Inferring the $\tilde{d}$-tuple after calculating the absolute values $|z_j|$ (which is trivially representable as a ReLU network as $\max(0,z_j) + \max(0,-z_j)$) can be done using the ``selector'' circuit introduced in `\cite{arora2017generalization}. Namely, by Lemma 3 there, there exists a ReLU network with $O(k)$ non-zero weights that takes $|z_j|$ as input and outputs $k$ numbers $(b_{j,1}, b_{j,2}, \dots, b_{j,k})$, s.t. $b_{j,i_j} = 1$ and $b_{j,l} = 0, j \neq i_j$ with probability $1-\delta$ over the randomness of $z_j, j \in [\tilde{d}]$.  

%\Anote{Needs a slack at the end of the intervals.}   
Since we care about $G(\nu)$ and $\tilde{G}(\nu)$ being close in total variation distance only, we can focus on the case where all $z_j$ are such that $b_{j,i_j} = 1$ and $b_{j,l} = 0, j \neq i_j$ for some indices $i_j$. 

We wish to now ``turn on'' the memorized weights for the corresponding block in the partition. To do this, we first pass the calculated $\tilde{d}$-tuple through a network which interprets it as a number in $k$-ary and calculates it's equivalent decimal representation. This is easily implementable as a ReLU network with $O(\tilde{d} k)$ weights calculating $L = \sum_{j=1}^{\tilde{d}} k^j \sum_{l=1}^k b_{j,l} i_j $. Then, we use use a simple circuit with $O(m)$ non-zero weights to output $m$ numbers $(B_1, B_2, \dots, B_m)$, s.t. $B_L=1$ and $B_i = 0, i \neq L$ (implemented in the obvious way). The subnetwork of $B_i$ will be responsiple for the $i$-th memorized sample. 

Namely, we attach to each coordinate $B_i$ a curcuit with fan-out of degree $d$, s.t. the weight of edge $j \in [d]$ is $x^*_{i,j}$. Let's denote these outputs as $F_{i,j}, i \in [m], j \in [d]$ and let $\tilde{F}_i: i \in [d]$ be defined as $\tilde{F}_i = \sum_{j=1}^m F_{i,j}$. It's easy to see since $B_i = 0, i \neq L$ that $\tilde{F}_i = x^*_{L,i}$.     

Finally, the operation $\circledast$ can be trivially implemented using additional $d$ weights: we simply connect each output $i$ either to 
$z_{\frac{i}{\lfloor \frac{d}{\tilde{d}}\rfloor}}$ if $i \equiv 0 (\mbox{ mod } \lfloor \frac{d}{\tilde{d}} \rfloor)$ or to $\tilde{F}_{i}$ otherwise. 

Adding up the sizes of the individual components, we see the total number of non-zero weights is $O(md)$, as we wanted. 
 
\end{proof}

\begin{figure}[H]

\label{f:circuit}
\centering
\includegraphics[width=0.7\textwidth]{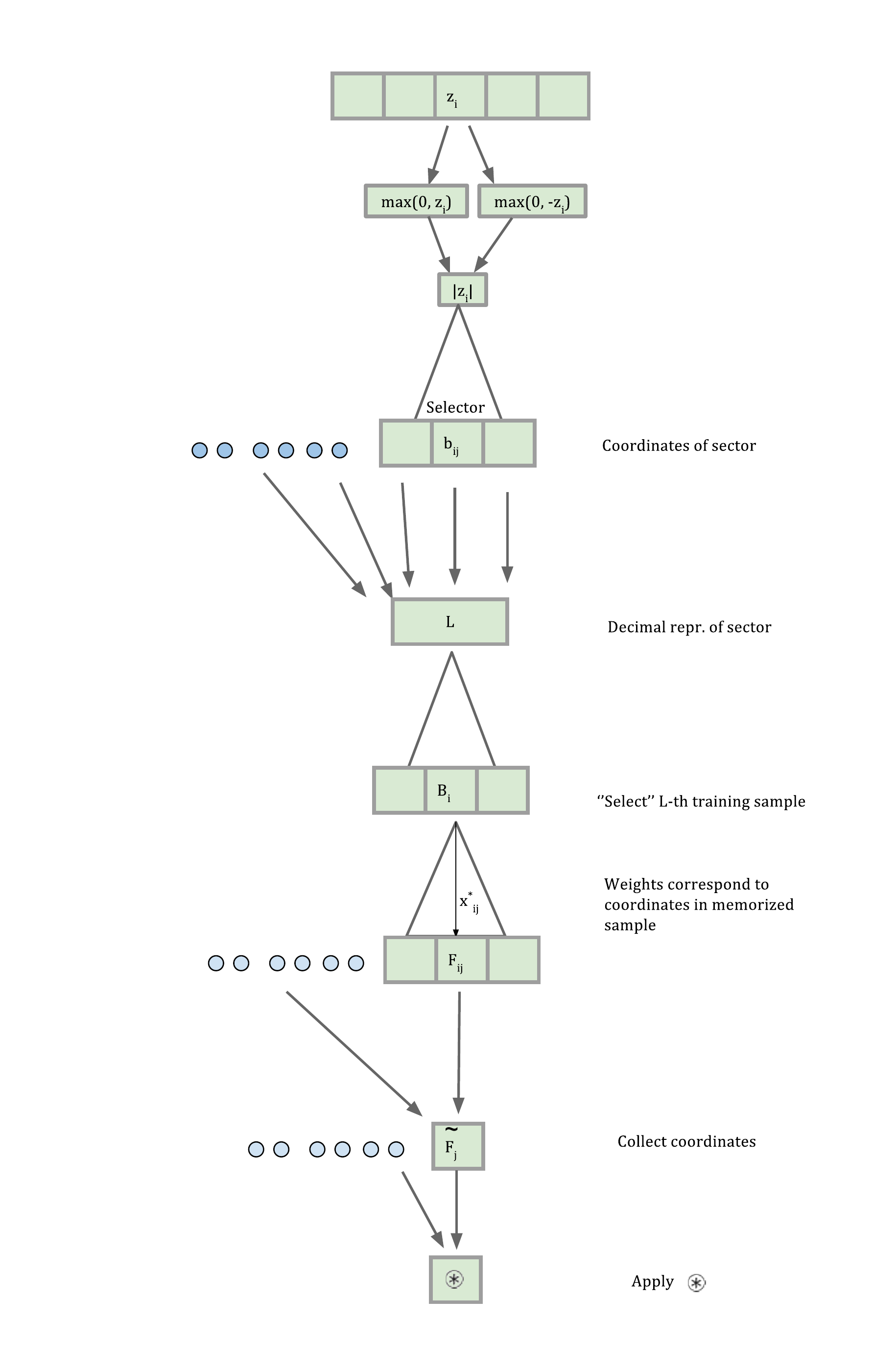}
\caption{A ReLU implementation of a memorizing generator}
\end{figure}